\documentclass{article}
\usepackage{amsthm}
\newtheorem{definition}{Definition}

\usepackage{PRIMEarxiv}

\usepackage{balance} 
\usepackage{graphicx}
\usepackage{color}
\usepackage{algpseudocode}
\usepackage[linesnumbered,ruled,vlined]{algorithm2e}
\usepackage{multirow}
\usepackage{makecell}
\usepackage{caption}
\usepackage[utf8]{inputenc}

\usepackage[bb=dsserif]{mathalpha}
\usepackage[bb=dsserif]{mathalpha}
\usepackage{bm}

\newcommand\IndicatorFunction{%
    \mathbb{1}
    }

\usepackage[T1]{fontenc}    
\usepackage{hyperref}       
\usepackage{url} 
\usepackage{booktabs}       
\usepackage{amsfonts}       
\usepackage{nicefrac}       
\usepackage{microtype}      
\usepackage{lipsum}
\usepackage{fancyhdr}       
\graphicspath{{media/}}     

\usepackage{comment}
\usepackage{amsmath}
\newtheorem{theorem}{Theorem}

\SetKwComment{Comment}{/* }{ */}
\SetKwFunction{BuildSFITable}{BuildSFITable}
\SetKwFunction{FirstSafeInterval}{FirstSafeInterval}
\SetKwFunction{Heuristic}{Heuristic}
\SetKwFunction{PrunedNeighbors}{PrunedNeighbors} 
\SetKwFunction{TravelTime}{TravelTime}
\SetKwFunction{GeomConflict}{GeomConflict}
\SetKwFunction{UpdatePath}{updatePath}
\SetKwFunction{DirectionSector}{DirectionSector}
\SetKwFunction{Neighbors}{Neighbors}
\SetKwFunction{RunSFIPP}{RunSFIPP-ST}
\SetKwFunction{sort}{sort}
\SetKwFunction{shuffle}{random\_shuffle}
\SetKwFunction{CreateStartNode}{CreateStartNode}
\SetKwFunction{BacktrackPath}{BacktrackPath} 
\SetKwFunction{NextWaitIntervalExists}{NextWaitIntervalExists}
\SetKwFunction{NextWaitInterval}{NextWaitInterval} 
\SetKwFunction{GeoConflict}{GeoConflict}

\title{Multi UAVs Preflight Planning in a Shared and Dynamic Airspace}

\author{
  Amath Sow, Mariusz Wzorek, Daniel de Leng, Mattias Tiger, Fredrik Heintz\\
  Linköping University, Sweden \\
  \texttt{\{amath.sow, mariusz.wzorek, daniel.de.leng, mattias.tiger, fredrik.heintz\}@liu.se} \\
   \And
  Mauricio Rodriguez, Christian Rothenberg \\
  Universidade Estadual de Campinas, Brazil  \\
  \texttt{\{m272321, chesteve\}@dac.unicamp.br} \\
  \And
  Fabíola M. C. de Oliveira \\
  Universidade Federal do ABC, São Paulo, Brazil  \\
  \texttt{fabiola.oliveira@ufabc.edu.br} 
}

\begin{document}
\maketitle

\begin{abstract}
Preflight planning for large-scale Unmanned Aerial Vehicle (UAV) fleets in dynamic, shared airspace presents significant challenges, including temporal No-Fly Zones (NFZs), heterogeneous vehicle profiles, and strict delivery deadlines. While Multi-Agent Path Finding (MAPF) provides a formal framework, existing methods often lack the scalability and flexibility required for real-world Unmanned Traffic Management (UTM). We propose DTAPP-IICR: a Delivery-Time Aware Prioritized Planning method with Incremental and Iterative Conflict Resolution. Our framework first generates an initial solution by prioritizing missions based on urgency. Secondly, it computes roundtrip trajectories using SFIPP-ST, a novel 4D single-agent planner (Safe Flight Interval Path Planning with Soft and Temporal Constraints). SFIPP-ST handles heterogeneous UAVs, strictly enforces temporal NFZs, and models inter-agent conflicts as soft constraints. Subsequently, an iterative Large Neighborhood Search, guided by a geometric conflict graph, efficiently resolves any residual conflicts. A completeness-preserving directional pruning technique further accelerates the 3D search. On benchmarks with temporal NFZs, \text{DTAPP-IICR} achieves near-100\% success with fleets of up to 1,000 UAVs and gains up to 50\% runtime reduction from pruning, outperforming batch Enhanced Conflict-Based Search in the UTM context. 
Scaling successfully in realistic city-scale operations where other priority-based methods fail even at moderate deployments, \text{DTAPP-IICR} is positioned as a practical and scalable solution for preflight planning in dense, dynamic urban airspace.
\end{abstract}


\section{Introduction}
The rapid growth of Unmanned Aerial Vehicle (UAV) operations in low-altitude airspace has significantly increased airspace density and complexity. UAVs are now deployed across a wide range of civilian and governmental applications, including package delivery, search and rescue, and surveillance~\cite{lastMilesDrone, SearchAndRescue, DroneAgri, Surveillance}. Global adoption is accelerating, with millions of UAVs expected to operate concurrently across the globe in the coming decades~\cite{UTM-USA1, UTM-USA, UTM-CHINA, SESAR}. This complexity highlights the need for Unmanned Traffic Management (UTM) systems to ensure conflict-free navigation in dense airspace.

In the context of UTM, UAVs are allowed to operate below 120 meters (400 feet), as stipulated by the European Union Aviation Safety Agency (EASA)~\cite{EASA}. The U-Space framework~\cite{EASA1} is specifically designed to ensure the safe integration of UAVs into the airspace, particularly in urban and complex environments. To achieve this objective in a shared airspace, an effective \emph{preflight planner} is essential. In this study, we focus on a service delivery scenario in which quadcopter UAVs transport goods from centralized hubs to target locations such as homes or hospitals. UAV Service Providers~(Hubs) periodically submit operation requests to the UTM system. The central controller processes these requests and approves flights based on compliance with operational requirements and airspace availability. Accepted flights are then authorized and coordinated within the UTM system to ensure safe and efficient operations.

Within this controlled urban environment, we model the problem as a Multi-Agent Path Finding (MAPF) variant in a 4D space (3D + time). Missions are prioritized by urgency and must strictly respect temporal No-Fly Zones (NFZs). Unlike standard point-to-point pathfinding, these operations require executing a complete roundtrip (outbound, wait, and return). This multi-stage structure is similar to Multi-Agent Pickup and Delivery (MAPD)~\cite{MAPD-Ma}, which generalizes MAPF to handle sequential tasks. However, our formulation differs by treating the entire roundtrip as a single atomic mission that must be planned holistically under strict 4D regulatory constraints, rather than as a sequence of tasks. Furthermore, existing work in the UTM context is largely limited to traditional batch methods, such as Conflict-Based Search (CBS) and Enhanced CBS (ECBS)~\cite{Extended-MAPF}, which lack the scalability and flexibility required for such dense and dynamic environments.

To address these challenges, we propose a new two-tier planning framework tailored to dynamic and high-density UAV operations. Our contributions are described as follows:

\begin{enumerate}
    \item  A Novel 4D Single-Agent Planner: We propose Safe Flight Interval Path Planning with Soft and Temporal Constraints (\text{SFIPP-ST}), which natively supports heterogeneous UAV profiles (varying start times, speeds, and sizes), strictly enforces temporal NFZs, and models inter-agent conflicts using soft, penalizing constraints;
    \item  A Scalable Prioritized Planning Framework: We introduce the Delivery-Time Aware Prioritized Planning method with Incremental and Iterative Conflict Resolution (\text{DTAPP-IICR}), an urgency-aware preflight planner for heterogeneous fleets. This framework employs an iterative conflict-resolution approach based on Large Neighborhood Search (LNS), enhancing scalability in dense airspace and replanning the most critical agents guided by a geometrical heuristic;
    \item Efficiency via Directional Pruning: To accelerate the path search in a 3D grid with 26-connected voxels, we develop a completeness-preserving directional pruning strategy that reduces the branching factor by prioritizing neighbors whose direction aligns with the goal vector, with a fallback mechanism that guarantees search completeness.
\end{enumerate}

Supported by its theoretical guarantees and validated through experimental evaluation, our approach offers a robust, scalable, and collision-free planning framework for high-density UAV deployments in shared, dynamic urban airspace. The implementation is made reproducible and available as an open-source project \footnote{URL: \url{https://github.com/amathsow/4DPlanning}}.

\section{Related Work} \label{sec:related-work}
MAPF is a core problem in multi-agent systems, modeling collision-free coordinated navigation of agents in shared environments~\cite{MAPF}.
Research in this area includes optimal and bounded-suboptimal solvers. Optimal solvers such as CBS~\cite{CBS} and Improved CBS (ICBS) \cite{ICBS} guarantee to find the least cost solution but often struggle to scale in dense environments, while bounded-suboptimal variants like ECBS~\cite{ECBS} and Explicit Estimation CBS~(EECBS)~\cite{EECBS} relax optimality to improve scalability. Prioritized Planning (PP)~\cite{PP} improves scalability by sequentially planning agents according to predefined orders, while Priority-Based Search (PBS)~\cite{ma2019searching} dynamically explores priority orderings. While these approaches are formulated for general graph problems, they do not fully address UTM-specific challenges.

Anytime approaches, particularly those based on LNS~\cite{LNS, MLLNS, LNS2, Balance24, ADDRESS}, have been shown to scale well by iteratively destroying and repairing portions of candidate solutions. LNS2~\cite{LNS2}, in particular, has demonstrated significant success in grid benchmarks by utilizing Safe Interval Path Planning with Soft constraints~(SIPPS)~\cite{LNS2} to resolve conflicts. These ideas inspire our framework. However, standard LNS2 typically rely on discrete-time vertex/edge conflicts, homogeneous agents, and synchronous start times. Our framework, DTAPP-IICR, adapts the LNS2 architecture but fundamentally redesigns the conflict model and repair operators. Unlike the discrete operator, our approach identifies conflicts using continuous-time geometric envelopes to handle heterogeneous agent sizes and speeds. Furthermore, while LNS2 utilizes SIPPS with reservation tables, our single-agent planner, SFIPP-ST, constructs safe intervals directly from dynamic NFZ constraints and geometric conflict checking, as interval reuse is invalid with heterogeneous velocities.

In UTM-specific contexts, only a few studies have directly addressed preflight multi-UAV planning. Batch extensions of CBS and ECBS~\cite{Extended-MAPF} support heterogeneous agents but assume symmetric roundtrips and cannot accommodate temporal NFZs. More recently, Rolling Continuous-time CBS~(RCCBS)~\cite{RCCBS} extended MAPF to continuous time by decomposing tasks into windowed batches and resolving conflicts through geometric overlap detection between agents. While RCCBS enables dynamic task arrivals and continuous-space planning, it remains constrained to semi-urban environments with static airspace conditions. Other studies have considered heterogeneous edges and roundtrips~\cite{MAPF-HR}, or combined takeoff scheduling with speed adjustment~\cite{TakeoffSpeedAdjutement}, to improve realism, but these approaches still rely on static maps.
In contrast, our framework integrates urgency-aware prioritization, temporal NFZ handling, and LNS-style conflict repair to provide scalable, regulation-compliant preflight planning for urban UAV operations. A comparison of the algorithms is presented in Table~\ref{tab:feature_comparison}.

\begin{table*}[!th]
    \centering
    \caption{Comparison of MAPF methods based on key features relevant to urban air mobility: Heterogeneous Agents~(HA), Conflict Types, Roundtrip, Temporal NFZs, and Urban Environment~(Urban Env).}
    \label{tab:feature_comparison}
    \resizebox{\textwidth}{!}{%
        \begin{tabular}{lccccc}
            \toprule
            \textbf{Method} & \textbf{HA} & \textbf{Conflict Types} & \textbf{Roundtrip} & \textbf{Temporal NFZs} & \textbf{Urban Env} \\
            \midrule
            Batch CBS/ECBS & \checkmark & Geometric & Symmetric outbound/return & \(\times\) & Semi-urban (Japan study) \\
            PP/PBS & \(\times\) & Vertex/Edge & \(\times\) & \(\times\) & Grid map \\
            MAPF-LNS & \(\times\) & Vertex/Edge & \(\times\) & \(\times\) & Grid map \\
            \textbf{DTAPP-IICR (Ours)} & \textbf{\checkmark} & \textbf{Geometric} & \textbf{Asymmetric (realistic UTM)} & \textbf{\checkmark} & \textbf{Urban (dense airspace)} \\
            \bottomrule
        \end{tabular}%
    }
\end{table*}

\section{Problem Definition}
\label{sec:problem-formalization}
We consider a set of UAVs $\mathcal{U}=\{U_1,\dots,U_k\}$ executing delivery operations. The environment is discretized as an undirected $26$-connected cubic grid represented as a graph \(G=(V, E)\); vertices are unit cubic voxels \((V)\) and edges \((E)\) connect all 6 face-, 12 edge-, and 8 corner-adjacent voxels, allowing diagonal motion.

\begin{definition}[UAV model] \label{def:uav-model}
  Each UAV \(U_i\) is described by  
  \begin{itemize}
  \item a hub (start and return) vertex \(s_i \in V\) and a delivery vertex
        \(g_i \in V\);
  \item a departure time \(t^{\mathrm{init}}_i > 0\);
  \item a constant cruise speed \(v_{u_i} > 0\);
  \item a physical UAV radius \(r_{u_i} > 0\) with spherical UAV profile.
  \end{itemize}
\end{definition}

\begin{definition}[Roundtrip path] \label{def:roundtrip}
Each mission is a \emph{roundtrip} composed of an outbound path, an
on-site waiting period (e.g., package handover), and a return path.
A path for \(U_i\) is a sequence of space–time tuples:
\begin{equation}
p_i = \bigl[(v_{i,0}, t_{i,0}),\,\dots,\, (v_{i,\tau}, t_{i,\tau}),\,\dots,\,
(v_{i,\ell_i}, t_{i,\ell_i})\bigr],
\end{equation}
where \(v_{i,0}=v_{i,\ell_i}=s_i\); \(\ell_i\) is the path length and \(t_{i,\ell_i}\) is the arrival time of UAV \(U_i\). There exists a unique index \(\tau\) that marks the \emph{first} time the UAV reaches its delivery vertex, i.e., \(v_{i,\tau}=g_i\). For some \(w \ge 0\), the consecutive tuples \(v_{i,\tau},\dots,v_{i,\tau+w}\) are identical, representing a \(w\)-timestep waiting period at \(g_i\). The trajectory between consecutive distinct tuples is traversed at the constant UAV speed \(v_{u_i}\), forming the outbound and return segments.
\end{definition}
\begin{definition}[Temporal NFZs] \label{def:nfzs}
The airspace may contain a finite set \(\mathcal{F}\) of NFZs. Each NFZ is characterized by the pair  \((R_f,[t^{\mathrm{start}}_f,t^{\mathrm{end}}_f])\) with  \(R_f\subseteq V\) being a spatially connected subset of voxels representing the restricted region and \([t^{\mathrm{start}}_f,t^{\mathrm{end}}_f]\subset \mathbf{R}_{\ge 0}\) is its active time interval. Once activated at \(t^{\mathrm{start}}_f\), an NFZ remains in force until \(t^{\mathrm{end}}_f\), after which it is removed from the operational airspace. We explicitly \emph{exempt} a UAV’s hub and delivery vertices from NFZ restrictions, reflecting regulatory practice that allows authorized take-off and landing inside controlled areas.
\end{definition}
\paragraph{\text{Solution validity constraints:}}
\begin{itemize}
  \item \text{NFZ compliance.}
        Let \(pos_i(t)\) denote the position of UAV \(U_i\) at time \(t\). For any time \(t\),
        \begin{equation}
          pos_i(t)\notin
          F(t):=\!\!\bigcup_{f\in\mathcal{F}}
          \!\!\bigl\{\,v\in R_f \,\bigl|\, t\in[t^{\mathrm{start}}_f,t^{\mathrm{end}}_f]\bigr\}
        \end{equation}
        unless \(pos_i(t)\in\{s_i,g_i\}\).
  \item \text{Collision avoidance.}  
        In this paper, conflicts are classified as pursuit, head-on, or intersection~\cite{Extended-MAPF} as shown in Fig.~\ref{fig:conflict}.

        \begin{figure}[t]
            \centering
            \includegraphics[width=0.5\linewidth]{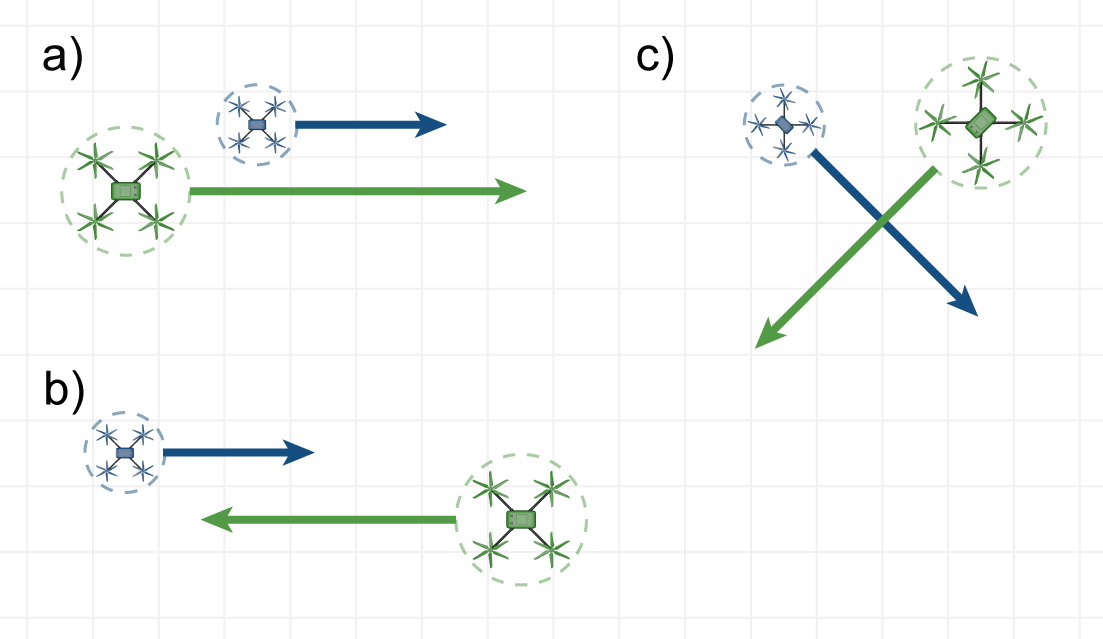}
            \caption{Types of conflicts between two agents of different sizes and speeds. a) Pursuit: agents move closely towards the same direction; b) Head-on: agents move into opposite directions; c) Intersection: agents cross paths~\cite{Extended-MAPF}.}
            \label{fig:conflict}
        \end{figure}

        Additionally, to ensure physical collision avoidance between UAVs, we
        introduce a safety buffer \(\gamma > 0\). Thus, for any two UAVs \(U_i\)
        and \(U_j\), the following must always hold:
        \begin{equation}
        \forall t,\quad \text{dist}\big(pos_i(t), pos_j(t)\big) > r_{u_i} + r_{u_j} + \gamma.
        \end{equation}

\end{itemize}

\paragraph{\text{Objective}}
Given the set $\mathcal{U}$ of UAVs, a \emph{solution} is a set of
collision-free, NFZ-respecting roundtrip paths
\(P=\{p_i \mid U_i\in\mathcal{U}\}\).
We minimize the flowtime (total flight time
at constant speed):
\begin{equation}
\min_{P}\;\; \sum_{U_i\in\mathcal{U}} T_i
\;=\;
\min_{P}\;\; \sum_{U_i\in\mathcal{U}} \bigl(t_{i,\ell_i} - t_{i,0}\bigr).
\end{equation}

\section{DTAPP Framework} 
\label{sec:dtapp}
We now introduce the \emph{DTAPP} (Delivery-Time Aware Prioritized Planning) framework for scalable, conflict-free multi-UAV preflight planning in dynamic airspace. \text{DTAPP} adopts a prioritized planning paradigm, where UAVs are ordered by delivery urgency (start time \(t^{\mathrm{init}}\)). To support this, we develop \emph{SFIPP-ST}, a 4D single-agent planner that computes NFZ-compliant, collision-free roundtrip trajectories while respecting heterogeneous UAV dynamics. The architecture of the system is described in Fig.~\ref{fig:archi}.

\begin{figure*}[t]
    \centering
    \includegraphics[width=\linewidth]{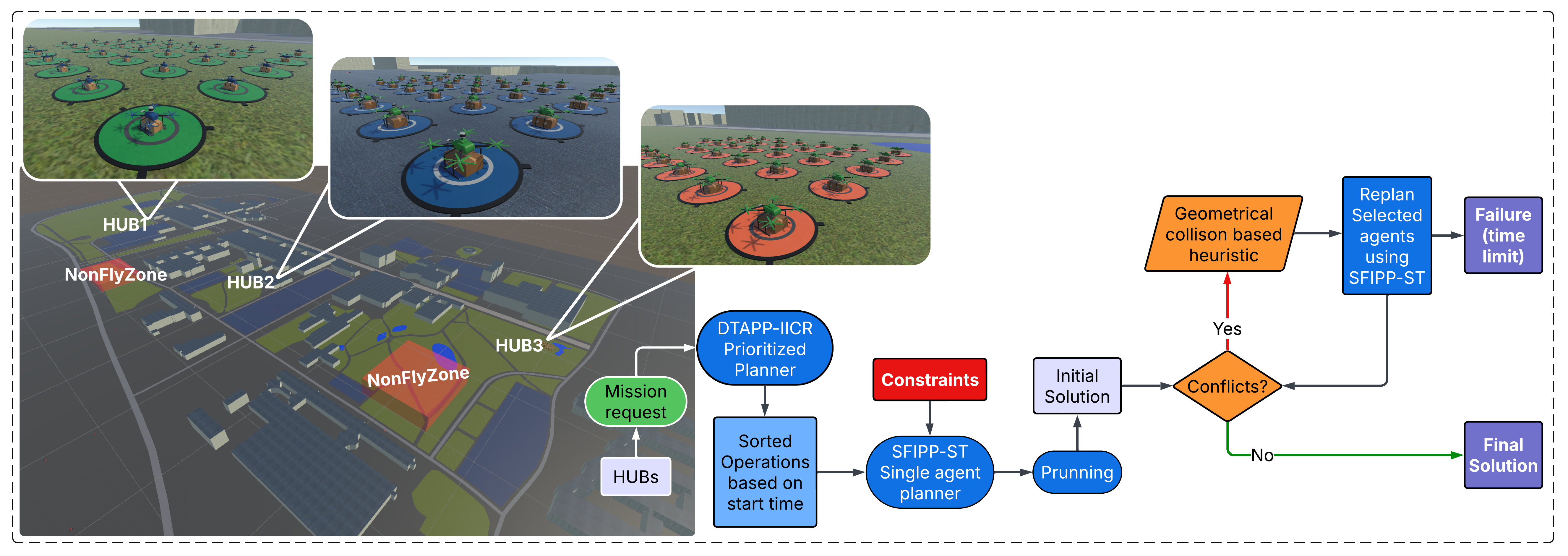}
    \caption{Overview of Preflight Planning in UTM. We have three hubs that send periodic operation requests to the central controller. We use \text{DTAPP-IICR} as a prioritized planner, where operations are sorted based on delivery urgency, $t^{\mathrm{init}}$. We use \text{SFIPP-ST} as the single-agent planner, which receives the UAV profile $(t^{\mathrm{init}}, v_{u_i}, r_{u_i})$, precomputed hard and regulatory constraints, and computes a 4D roundtrip trajectory from the hub to the delivery location and back. We add pruning during neighbor expansion to reduce the search. The initial solution might have soft conflicts, and we iteratively replan the conflicting agents until the solution is free of conflict or the time limit is reached.}
    \label{fig:archi}
\end{figure*}
\subsection{Single Agent Planning with Soft and Temporal Constraints}
\noindent
We extend the path planning framework from \text{PMDO} (Pathfinding with Mixed Dynamic Obstacles) as introduced in \cite{LNS2}, by integrating additional constraints such as \emph{temporal NFZs} and UAV profiling (\emph{different start times, speeds, and sizes}).

Formally, the planner must compute a roundtrip path \(p\) for a given UAV while avoiding three categories of restrictions: \textbf{hard obstacles} \(O_h\), which are static obstacles in the environment that must be strictly avoided; \textbf{temporal NFZs} \(\mathcal{F}\), which impose regulatory constraints that must be strictly respected, as formally defined by the NFZ compliance condition in Equation (2); and \textbf{soft constraints} \(O_s\), which represent other UAV trajectories. In this case, \text{SFIPP-ST} penalizes potential conflicts with \(O_s\) but does not strictly forbid them, allowing iterative conflict resolution to refine the solution.

\subsection{SFIPP-ST}
SFIPP-ST is a single-agent planner for 4D trajectory planning in UTM systems. It extends SIPPS~\cite{LNS2} by incorporating temporal NFZs, heterogeneous UAV profiles, and continuous geometric conflict modeling, making it suitable for UTM.
\begin{definition}[Safe Flight Interval (SFI)] \label{def:sfi}
A SFI for a vertex \( v \in V \) is defined as a contiguous time interval \( [t_{\text{start}}, t_{\text{end}}) \) during which a UAV can safely traverse or wait at \( v \) without: (1) Colliding with any hard obstacle \( O_h \). (2) Violating any active temporal NFZ \( \mathcal{F} \).
We define the indicator function for NFZ exclusion as:
\[
\IndicatorFunction_{\mathcal{F}}(v,t) = 
\begin{cases}
1 & \text{if } \exists (R_f, [t^{\mathrm{start}}_f, t^{\mathrm{end}}_f]) \in \mathcal{F},\, v \in R_f,\, t \in [t^{\mathrm{start}}_f, t^{\mathrm{end}}_f] \\
0 & \text{otherwise}.
\end{cases}
\]

The set of SFIs at vertex \( v \) is given by:
\begin{align}
T(v)
  &= \bigl\{ [t_{\text{start}}, t_{\text{end}}) \subseteq \mathbf{R}_{\ge 0}
     \,\big|\, 
     \forall t \in [t_{\text{start}}, t_{\text{end}}), \nonumber\\
  &\quad v \notin O_h \text{ and } \IndicatorFunction_{\mathcal{F}}(v,t) = 0
     \bigr\}.
\end{align}
That is, SFIs are all maximal contiguous time intervals where the vertex \( v \) is free from both hard obstacles and active NFZs.
\end{definition}
\begin{definition}[Search Node Definition] \label{def:snd}
SFIPP-ST represents each search node as a tuple:
\[
n = \big( v,\, I,\, i,\, g(n),\, h(n),\, c(n),\, r_{u_i},\, v_{u_i} \big),
\]
where: \( v \in V \): the current vertex location; \( I = [t_{\text{start}}, t_{\text{end}}) \): the associated SFI free from \( O_h \) and \( \mathcal{F} \); \( i \): the identifier of \( I \) in the precomputed SFI table \( T(v) \); \( g(n) \): the earliest arrival time at node \( n \); \( h(n) \): the heuristic cost-to-go; \( c(n) \): the cumulative conflict cost, representing the estimated number of geometric collisions with other UAV paths \( O_s \); \( r_{u_i} \) and \( v_{u_i} \) are respectively the radii and the speed of UAV \( U_i \).
The overall cost function for node \( n \) is defined as:\(f(n) = g(n) + h(n).\)
\end{definition}
\paragraph{\textbf{Main Algorithm:}}
Algorithm~\ref{alg:sfippst} begins by constructing the SFI table $T$ for all voxels $V$ in the 3D grid. Each $T(v)$ lists contiguous time intervals where $v$ is free from both hard obstacles ($O_h$) and active NFZs ($\mathcal{F}$). This preprocessing step reduces the 4D search problem into a set of time-constrained intervals, enabling efficient exploration (Line~1).
We then query the earliest safe interval $I_0$ at the start vertex $s$ valid at or after $t^{\mathrm{init}}_i$; if none exists, return \texttt{FAILURE} (Lines~2--4). Otherwise, we set the actual departure time to $t_s = \max(t^{\mathrm{init}}_i, I_0.\text{start})$, ensuring that the UAV departs no earlier than its mission schedule (Line~5). At this point, we create the root node $n_{\text{root}}$ at $s$ using the UAV profile (Line~6). The node is inserted into the OPEN list, with the CLOSED set initialized as empty (Line~7). Nodes in OPEN are prioritized by ascending soft collision cost c; ties are broken by the total cost $f$, ensuring that SFIPP-ST first minimizes conflicts and then travel time.  
At each iteration, the algorithm pops a node $n$ from $OPEN$ (Line~9). If $n$ corresponds to the goal vertex $g$, we reconstruct its trajectory using \textsc{BacktrackPath}$(n)$, which traces parent nodes back to the root (Lines~10--11). Otherwise, $n$ is expanded using a pruning strategy that keeps neighbors $v'$ aligned with the goal direction (Line~12). For each promising neighbor $v'$ and each interval $I'\in T[v']$ (Line~13), the travel time $\delta$ is computed and the earliest feasible arrival $t'$ (Line~14--15). If $t'\notin I'$, the interval is discarded (Line~16--17). Otherwise, we evaluate geometric conflicts along the continuous segment $[n.v\!\to\!v']$ over $[g(n), t']$ checking against soft obstacles $O_s$ using radius $r_{u_i}$ to obtain the cumulative conflict cost $c'$ (Line~18). We estimate $h'$ as the minimum travel time from $v'$ to $g$ at speed $v_{u_i}$ (Line~19). A new node $n'$ is created and inserted into $OPEN$ (Lines~20--21).
In addition to movement, the algorithm also considers \emph{waiting} at the current vertex $n.v$, since this action may be required when future safe intervals exist. If a subsequent safe interval $I_{\text{wait}}$ is available after $n.I$ (Line~22), the wait time is set to $t_{\text{wait}} = I_{\text{wait}}.\text{start}$ (Line~24), and then we compute geometric conflicts for $[n.v\!\to\!n.v]$ over $[g(n), t_{\text{wait}}]$ to get $c_{\text{wait}}$ (Line~25). We estimate $h_{\text{wait}}$, create a wait node $n_{\text{wait}}$, and insert it into $OPEN$ (Lines~26--27).
The algorithm terminates with a conflict-minimizing path or \texttt{FAILURE} if $OPEN$ is exhausted (Line~28).

\begin{algorithm}[h]
\caption{\textsc{SFIPP-ST}}
\label{alg:sfippst}

\KwIn{Graph $G=(V,E)$, start $s$, goal $g$, UAV profile $(t^{\mathrm{init}}_i, v_{u_i}, r_{u_i})$, hard obstacles $O_h$, temporal NFZs $\mathcal{F}$, soft obstacles $O_s$}
\KwOut{Path $p$ or \texttt{FAILURE}}

$T \gets \BuildSFITable(V, O_h, \mathcal{F}, r_{u_i}, v_{u_i})$\;
$I_0 \gets \FirstSafeInterval(T[s], t^{\mathrm{init}}_i)$\;
\If{$I_0 = \emptyset$}{\Return \texttt{FAILURE}\;}
$t_s \gets \max(t^{\mathrm{init}}_i, I_0.\text{start})$\;
$n_{\text{root}} \gets \CreateStartNode(s, I_0, i, t_s, r_{u_i}, v_{u_i})$\;

$OPEN \gets \{ n_{\text{root}} \}$; $CLOSED \gets \emptyset$\; \tcp{Order by lowest $c(n)$; tie-break by $f(n)$}

\While{$OPEN \neq \emptyset$}{
  $n \gets \operatorname{OPEN.pop}()$\;
  \If{$n.v = g$}{\Return \BacktrackPath($n$)\;}
  
  \ForEach{$v' \in \PrunedNeighbors(n.v, g)$}{
    \ForEach{$I' \in T[v']$}{
      $\delta \gets \TravelTime(n.v, v', v_{u_i})$\;
      $t' \gets \max(g(n) + \delta, I'.\text{start})$\;
      \If{$t' \notin I'$}{\textbf{continue}\;}
      
      $c' \gets c(n) + \GeoConflict\big([n.v\!\to\!v'],\, [g(n), t'],\, r_{u_i},\, O_s\big)$\;
      $h' \gets \Heuristic(v', g, v_{u_i})$\;
      
      $n' \gets (v',\, g=t',\, h=h',\, c=c',\, I')$\;
      $OPEN \gets OPEN \cup \{ n' \}$\;
    }
  }
  
  \If{\NextWaitIntervalExists($n.v, n.I, T$)}{
    $I_{\text{wait}} \gets \NextWaitInterval(n.v, n.I, T)$\;
    $t_{\text{wait}} \gets I_{\text{wait}}.\text{start}$\;
    
    $c_{\text{wait}} \gets c(n) + \GeoConflict\big([n.v\!\to\!n.v],\, [g(n), t_{\text{wait}}],\, r_{u_i},\, O_s\big)$\;
    $h_{\text{wait}} \gets \Heuristic(n.v, g, v_{u_i})$\;
    
    $n_{\text{wait}} \gets (n.v,\, g=t_{\text{wait}},\, h=h_{\text{wait}},\, c=c_{\text{wait}},\, I_{\text{wait}})$\;
    $OPEN \gets OPEN \cup \{ n_{\text{wait}} \}$\;
  }
}
\Return \texttt{FAILURE}\;
\end{algorithm}

\paragraph{\textbf{Directional Pruning for Neighbor Expansion:}}
In a 3D grid with 26 connected voxels, a brute-force expansion of all neighbors at each step is computationally expensive. To accelerate the search, we introduce a directional pruning heuristic that reduces the branching factor. The core idea is to prioritize expanding neighbors that represent clear, monotonic progress toward the goal. This is achieved while rigorously preserving the search completeness.

The heuristic works by comparing the normalized direction vector from the current location to a neighbor with the direction vector to the final goal. This vector, which we call the \emph{direction sector}, is computed as:
\[
S = (\text{sign}(\Delta x),\, \text{sign}(\Delta y),\, \text{sign}(\Delta z)) \in \{-1, 0, 1\}^3
\]
where $(\Delta x, \Delta y, \Delta z)$ is the displacement vector. For example, \(S = (1, 0, -1)\) implies a move along positive \(x\), no change in \(y\), and negative \(z\).

A neighbor $n$ is considered promising if its direction sector, $S_{\text{nb}}$, aligns with the goal's direction sector, $S_{\text{goal}}$, under conditions that favor direct paths:

\begin{enumerate}
    \item \text{XY-Plane Alignment}: The neighbor lies in the same general direction as the goal on the 2D plane. This includes:
    \begin{itemize}
        \item \text{Diagonal moves} that match the goal's quadrant, for instance, if the goal is north-east, ($S_{\text{goal}}.x=1, S_{\text{goal}}.y=1$), any neighbor in that quadrant is promising.
        \item \text{Cardinal moves} along an axis if the goal also lies primarily on that axis (e.g., moving straight north if the goal is directly north).
    \end{itemize}
    \item \text{Z-Axis Alignment}: The neighbor involves a direct vertical movement toward the goal's altitude ($S_{\text{nb}}.z = S_{\text{goal}}.z \neq 0$). This rule prioritizes changing flight levels, a common requirement in UAV navigation.
\end{enumerate}
To guarantee completeness, our approach includes two important exceptions: any adjacent neighbor that is the goal is always included, and the "wait" action is always an option. If these pruning rules result in no valid moves, the algorithm reverts to expanding the full 26-neighbor set, preserving global completeness. The full logic is detailed in Algorithm~\ref{alg:direction-prune}.

\begin{algorithm}[t]
\caption{\textsc{PrunedNeighbors}$(V.loc, g)$}
\label{alg:direction-prune}
\KwIn{Set of neighbors $V$, current location $loc$, goal location $g$}
\KwOut{Pruned neighbor set $\mathcal{N}(loc)$}

$S_{\text{goal}} \gets$ \DirectionSector{$loc, g$}\;
$\mathcal{N}(loc) \gets \{loc\}$\; \tcp{Always include wait action}

\ForEach{$n \in V$}{
    \uIf{$n = g$}{
        $\mathcal{N}(loc) \gets \mathcal{N}(loc) \cup \{n\}$\;
    }
    \Else{
        $S_{nb} \gets$ \DirectionSector{$loc, n$}\;
        
        \tcp{Check alignment conditions}
        \If{($S_{nb}.x = S_{\text{goal}}.x \land S_{nb}.y = S_{\text{goal}}.y$) \textbf{or}\\
            \hspace{1cm}($S_{nb}.x = S_{\text{goal}}.x \land S_{nb}.x \neq 0 \land S_{nb}.y = 0$) \textbf{or}\\
            \hspace{1cm}($S_{nb}.y = S_{\text{goal}}.y \land S_{nb}.y \neq 0 \land S_{nb}.x = 0$) \textbf{or}\\
            \hspace{1cm}($S_{nb}.z = S_{\text{goal}}.z \land S_{nb}.z \neq 0$)}{
            $\mathcal{N}(loc) \gets \mathcal{N}(loc) \cup \{n\}$\;
        }
    }
}
\If{$|\mathcal{N}(loc)| = 1$}{
    $\mathcal{N}(loc) \gets V \cup \{loc\}$\; \tcp{Fallback to all neighbors}
}
\Return $\mathcal{N}(loc)$\;
\end{algorithm}
\subsection{Theoretical Properties of SFIPP-ST}
We now establish two theoretical properties of \text{SFIPP-ST} planner: \emph{completeness} and \emph{optimality}. 
\begin{theorem}[Completeness]
SFIPP-ST is complete; it returns a valid path if one exists and \texttt{FAILURE} otherwise.
\end{theorem}

\emph{Proof.} The proof follows the completeness guarantees of SIPP~\cite{SIPP} and SIPPS~\cite{LNS2}. The modifications in SFIPP-ST preserve completeness:  

1. \text{Finite state space.} Each search node is a tuple $n = (v, I, \ldots)$, where $v \in V$ and $I \in T(v)$. Since $V$ is finite and $T(v)$ is precomputed as a finite set of Safe Flight Intervals, the search space is finite. Thus, infinite exploration is impossible.  

2. \text{Temporal NFZs as hard constraints.} NFZs are incorporated in preprocessing by BuildSFITable, which computes all $[t_a, t_b)$ where $v$ is free from $O_h$ and $\mathcal{F}$. NFZs can therefore be treated as time-dependent hard obstacles, which preserves the validity of the safe-interval graph explored by the search.  

3. \text{Completeness-preserving pruning.} The directional pruning strategy reduces branching, but includes a fallback: if pruning eliminates all neighbors, the algorithm reverts to expanding the full 26-connected set. Hence no feasible successor is ever permanently excluded.  

Since SFIPP-ST explores a finite state space, enforces NFZs as hard constraints, and uses pruning that preserves reachability, it is complete. 

\begin{theorem}[Optimality]
If a conflict-free path exists, SFIPP-ST returns the shortest such path with zero soft collisions.
\end{theorem}

\emph{Proof.} The proof extends the optimality argument of SIPPS:  

1. \text{Ordering.} Nodes in OPEN are ordered first by conflict cost $c(n)$, then by $f(n) = g(n) + h(n)$. Thus, all zero-conflict ($c=0$) nodes are expanded before any $c>0$ nodes. This guarantees that if a conflict-free path exists, it is found before any path with conflicts.  

2. \text{Admissible heuristic.} The heuristic $h(n)$ is defined as the straight-line travel time from $v$ to the goal at speed $v_u$. Since the Euclidean distance never overestimates the true travel distance in a grid, $h(n)$ is admissible and ensures A*’s optimality among equal-$c$ paths.

Therefore, SFIPP-ST always returns the shortest available path with zero soft collisions, if one exists. 

\subsection{DTAPP-IICR}
\label{subsec:dtapp-iicr}
We now introduce \emph{DTAPP-IICR}, which combines sequential roundtrip planning with iterative conflict resolution. 
\text{DTAPP-IICR}, presented in Algorithm~\ref{alg:dtapp-iicr}, begins by sorting UAVs based on $t^{\text{init}}_i$, ensuring urgent deliveries are planned first. Ties among UAVs with equal start times are broken randomly to avoid bias and deadlocks (Lines 1--3). Each UAV $U_i$ computes its roundtrip path using \textsc{PlanRoundTrip}, treating all previously planned paths $P$ as dynamic obstacles (Line 6). The resulting paths are stored in the solution set $P$ (Line 7), and any conflicts between UAV trajectories are recorded in the collision graph (Line 8).
The iterative conflict resolution phase (Lines 9-19) systematically eliminates conflicts until a globally feasible solution is achieved or the time limit is exceeded. At each iteration, the algorithm first identifies the set $S_c$ of collided UAVs in the solution set $P$ (Line 10). Then we remove all paths $P_{old} \subset P$ belonging to the selected UAVs in $S$ (Lines 11--12).
Each UAV $U_i$ in this prioritized subset is then replanned sequentially using the \textsc{PlanRoundTrip} function (Line 14), treating all remaining paths in $P$ (from non-selected UAVs) as fixed spatio-temporal obstacles. This ensures that newly computed paths avoid conflicts with UAVs not currently being replanned. 
If a UAV successfully finds a valid conflict-free path, it is immediately added back to the solution set $P$, making it a constraint for subsequent UAVs in the current iteration (Lines 15--16). However, if any UAV fails to find a feasible path during replanning, the algorithm performs a rollback operation: all previously backed-up paths $P_{\text{old}}$ are restored to $P$, and the current iteration terminates (Line 18). We update the number of collision counts (Line 19).
If there are no more collisions, we return the final path $P$, or we return failure otherwise (Lines 20--23).

\begin{algorithm}[t]
\caption{DTAPP-IICR}
\label{alg:dtapp-iicr}
\KwIn{UAV set $U=\{U_1,\dots,U_n\}$, 3D grid $\mathcal{G}$, temporal NFZs $F(t)$, time limit $T_{\max}$}
\KwOut{Conflict-free paths $P$ or FAILURE}

$U' \gets \mathrm{sortAsc}\bigl(U,\, t^{\text{init}}\bigr)$\;
\ForEach(\tcp*[f]{break tie}){group $G \subseteq U'$ with equal $t^{\text{init}}$}{
      shuffle $G$ within $U'$\;
}
$P \gets \emptyset$\;
\ForEach{$U_i \in U'$}{
    $p_i \gets \textsc{PlanRoundTrip}(U_i, \mathcal{G}, F(\cdot), \emptyset)$\;
    $P \gets P \cup \{p_i\}$\;
}
$\text{num\_conflicts} \gets \textsc{CountConflicts}(P)$\;

\While{$\text{num\_conflicts} > 0$ \textbf{and} $t < T_{\max}$}{
    $S_c \gets \textsc{ConflictHeuristic}(P, \text{num\_conflicts})$\;
    $P_{\text{old}} \gets \{p_i : U_i \in S\}$\;
    $P \gets P \setminus P_{\text{old}}$\;
    \ForEach{$U_i \in S_c$}{
        $p_i \gets \textsc{PlanRoundTrip}(U_i, \mathcal{G}, F(\cdot), P)$\;
        \uIf{$p_i$ valid}{$P \gets P \cup \{p_i\}$}
        \Else{$P \gets P \cup P_{\text{old}}$; \textbf{break}}
    }
    $\text{num\_conflicts} \gets \textsc{CountConflicts}(P)$\;
}
\uIf{$\text{num\_conflicts} = 0$}{\Return $P$}
\Else{\Return \textbf{FAILURE}}

\BlankLine
\SetKwProg{Fn}{Function}{:}{}
\Fn{\textsc{PlanRoundTrip}$(U_i, \mathcal{G}, F(\cdot), P)$}{
    $p_{\text{out}} \gets \textsc{SFIPP-ST}(s_i \to g_i, \mathcal{G}, F(\cdot), P)$\;
    $p_{\text{wait}} \gets \textsc{GenerateWait}(g_i, w_i)$\;
    $p_{\text{ret}} \gets \textsc{SFIPP-ST}(g_i \to s_i, \mathcal{G}, F(\cdot), P \cup \{p_{\text{out}}, p_{\text{wait}}\})$\;
    \Return $\textsc{Combine}(p_{\text{out}}, p_{\text{wait}}, p_{\text{ret}})$\;
}
\end{algorithm}
\subsection{Geometric Conflict-Based Heuristic}
\label{subsec:geom-neigh}
We adapt the neighborhood selection strategy from LNS2~\cite{LNS2} to the continuous 4D domain. Instead of discrete graph collisions, we maintain a \emph{geometric collision graph} $G_c=(V_c, E_c)$, where each vertex denotes a UAV and an edge $(U_i, U_j)\in E_c$ exists if the pair violates the minimum separation $\text{dist}(pos_i(t), pos_j(t)) \le r_i + r_j + \gamma$ at any time $t$.

Given a target neighborhood size $N$, we construct a replanning set $\mathcal{A}_s$ as follows. First, sample a vertex $v\in V_c$ with $\deg(v)>0$ and extract the connected component $G_c^0=(V_c^0,E_c^0)$ containing $v$. Then, following the logic of~\cite{LNS2}, we apply one of two expansion cases:

  \paragraph{Case 1 (Small conflict region, $|V_c^0|\le N$):}
  Initialize $\mathcal{A}_s\leftarrow V_c^0$ and expand it via spatial random walks to include nearby agents that may restrict repair. Until $|\mathcal{A}_s|=N$ or no new agents can be found:
  1) Sample $U_a\in \mathcal{A}_s$ uniformly and a random time index $t$ along its current path; 2) From the corresponding 3D location, perform a random walk over the 26-connected neighborhood (including wait); and 3) If the walk reaches a location where the required safety buffer overlaps with another UAV $U_b$, add $U_b$ to $\mathcal{A}_s$.

  \paragraph{Case 2 (Large conflict region, $|V_c^0|>N$):}
  Perform a random walk on the \emph{collision graph} $G_c^0$ itself, collecting distinct vertices until $|\mathcal{A}_s|=N$. This biases selection toward densely connected subgraphs, prioritizing UAVs engaged in tightly coupled conflicts.

\section{Experimental Evaluation}
\label{sec:experiment}
To evaluate the effectiveness, efficiency, and robustness of our proposed methods, we conduct experiments in two settings: a Monte Carlo Simulation (MCS) environment and a realistic city-scale map.  

In the MCS setup, the environment is a $100 \times 100 \times 10$ 3D grid with 1 m$^3$ voxels providing high-precision collision detection and path planning. The lower four levels (0–3) contain static obstacles (e.g., buildings, trees), while the upper six levels (4–9) provide open airspace, reflecting typical urban layering. UAV profiles are randomly sampled: start times in [1 s, 1000 s], radii in [0.5 m, 2.0 m], and speeds in [1.0, 5.0 m/s]. NFZs are activated over intervals randomly drawn from [100s, 500s], capturing the time-varying nature of restricted airspace. Across trials, we vary the number of UAV operations, obstacle densities, and NFZs to test scalability and robustness. The duration of package handover is fixed at $w=10$ and the neighborhood size used by \text{DTAPP-IICR} for iterative conflict resolution is fixed at $N=10$.

For external validation, we also evaluate our methods on a realistic city-scale map~(Fig.~\ref{fig:archi}) with designated UAV hubs, delivery destinations (e.g., homes, hospitals), and regulatory constraints. This case study demonstrates the applicability of our approach to practical urban air mobility scenarios.

As evaluation metrics, we measure: (1) \emph{Success rate}: percentage of instances solved within the time limit, and (2) \emph{Average runtime}: mean planning time, using the time limit for unsolved instances with their standard deviations. All experiments were run on an Intel Core i9-10900X CPU with 128GB RAM.
\paragraph{\textbf{Experiment 1: SFIPP-ST and Pruning Evaluation}}
In this experiment, we evaluate the effectiveness of the \text{SFIPP-ST} algorithm and the proposed pruning optimization. We compare Priority Planning (PP) with a modified version of \text{SFIPP-ST} (where soft obstacles are treated as hard obstacles) against \text{DTAPP-IICR}. Both solvers are tested with and without pruning across scenarios with 10, 50, and 100 agents in environments with 5\% obstacle density. For each configuration, we generate 20 instances and set a fixed time limit of 300s (5 minutes). Results are summarized in Table~\ref{tab:pruning1}.
\text{DTAPP-IICR} demonstrates superior scalability, maintaining a 100\% success rate across all agent counts, while PP without pruning fails completely at 100 agents. \text{DTAPP-IICR} achieves up to 15 times faster runtimes than PP, highlighting the fundamental advantage of soft collision handling over \text{SFIPP-ST}'s hard obstacle treatment. Our direction-based pruning provides significant speedup for both algorithms, with \text{DTAPP-IICR} achieving up to 34-50\% faster performance and PP with pruning remains viable at scale.
\begin{table}[t]
    \caption{Comparison of solver performance across varying agent counts (N). PP and DTAPP-IICR denote versions with pruning. (W/O) indicates no pruning.}
    \label{tab:pruning1}
    \centering
    \begin{tabular}{lccc}
        \toprule
        \textit{Solver} & \textit{N} & \textit{Success Rate (\%)} & \textit{Runtime (s)} \\
        \midrule
        PP (W/O)         & \multirow{4}{*}{10}  & \multirow{4}{*}{\textbf{100}} & 6.05 $\pm$ 2.04   \\
        PP               &   &  & 3.06 $\pm$ 1.09  \\
        DTAPP-IICR (W/O) &   &  & 2.31 $\pm$ 0.82  \\
        DTAPP-IICR       &   &  & \textbf{1.31 $\pm$ 0.57}   \\
        \midrule
        PP (W/O)         & \multirow{4}{*}{50}  & \multirow{4}{*}{\textbf{100}} & 115.81 $\pm$ 11.64 \\
        PP               &   &  & 54.53 $\pm$ 5.90  \\
        DTAPP-IICR (W/O) &   &  & 10.97 $\pm$ 1.32  \\
        DTAPP-IICR       &   &  & \textbf{7.23 $\pm$ 1.06}  \\
        \midrule
        PP (W/O)         & \multirow{4}{*}{100} & 0   & 300.0 $\pm$ 0.00 \\
        PP               &  & \textbf{100} & 223.22 $\pm$ 16.14 \\
        DTAPP-IICR (W/O) &  & \textbf{100} & 25.91 $\pm$ 2.15  \\
        DTAPP-IICR       &  & \textbf{100} & \textbf{13.09 $\pm$ 1.26} \\
        \bottomrule
    \end{tabular}
\end{table}

\paragraph{\textbf{Experiment 2: Scalability Evaluation}}
This experiment evaluates the performance of \text{DTAPP-IICR} as the number of operations increases.  We consider our previous environment with  5\% and 10\% obstacles, where the number of UAVs increases up to 1000. We implemented batch CBS and ECBS from \cite{Extended-MAPF} since the original code is unavailable and compared them against our \text{DTAPP-IICR} and PP. To get a fair comparison, we consider only the outbound path without temporal NFZs. The return path will then be symmetric to the outbound. The time limit is fixed at 480s (8min). Figure~\ref{fig:scalability} presents the results. While \text{DTAPP-IICR} maintains a 100\% success rate even with 1000 UAVs, batch CBS fails to scale beyond small instances due to its exponential complexity. Notably, PP also suffers from poor scalability, failing beyond 250 agents. This overly conservative treatment results in over-constrained planning, making feasible solutions unlikely in dense scenarios. ECBS remains more robust, sustaining higher success rates than CBS/PP thanks to its bounded-suboptimal, batch-search strategy. In contrast, \text{DTAPP-IICR} leverages soft constraints with iterative conflict resolution, allowing temporary conflicted paths that are later resolved, resulting in significantly higher success rates and reduced computation times. These results highlight \text{DTAPP-IICR}'s effectiveness for large-scale UAV traffic scenarios where flexibility and scalability are critical.

\begin{figure}[t]
    \centering
    \includegraphics[width=0.8\linewidth]{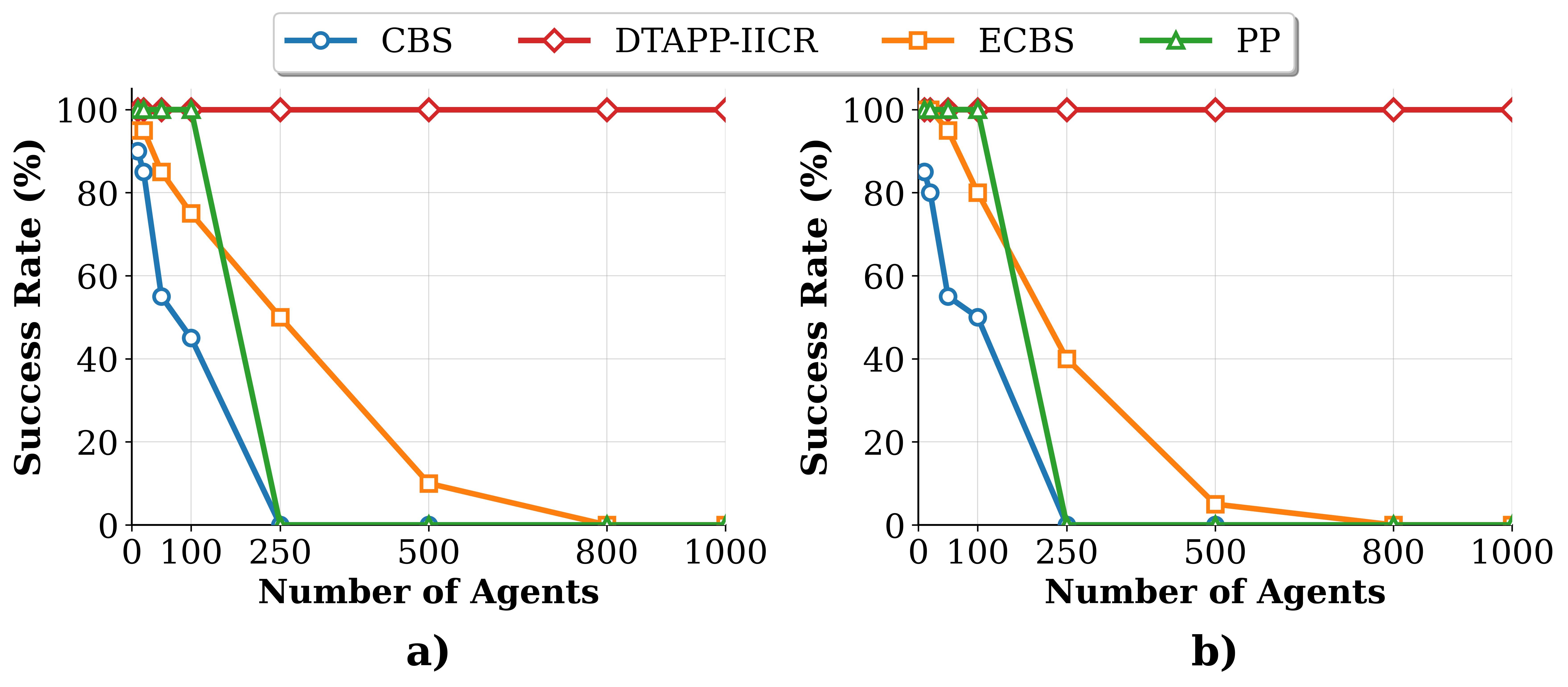}
    \caption{Scalability of \text{DTAPP-IICR} compared to PP, batch CBS/ECBS under varying obstacle densities. a) Using 5\% of obstacles. b) Using 10\% of obstacles.}
    \label{fig:scalability}
\end{figure}

\begin{table*}[h]
\centering
\small
\renewcommand{\arraystretch}{1.1}
\setlength{\tabcolsep}{5.0pt}
\caption{Success rate (\%) and runtime (s, mean ± std) for different algorithms across NFZ settings and agent counts.}
\begin{tabular}{l l c c c c c c c c}
\toprule
\textbf{NFZ} & \textbf{Agents} &
\multicolumn{2}{c}{\textbf{DTAPP-IICR}} &
\multicolumn{2}{c}{\textbf{DTAPP-IICR (W/O)}} &
\multicolumn{2}{c}{\textbf{PP}} &
\multicolumn{2}{c}{\textbf{PP (W/O)}} \\
\cmidrule(lr){3-4} \cmidrule(lr){5-6} \cmidrule(lr){7-8} \cmidrule(lr){9-10}
 & & Success & Runtime & Success & Runtime & Success & Runtime & Success & Runtime \\
\midrule
\multirow{5}{*}{0}
 & 10  & \textbf{100.0} & \textbf{14.5 $\pm$ 8.7}  & \textbf{100.0} & 28.8 $\pm$ 17.1  & 90.0 & >44.2 $\pm$ 14.7  & 90.0 & >67.7 $\pm$ 27.2  \\
 & 50  & \textbf{100.0} & \textbf{91.7 $\pm$ 14.7}  & \textbf{100.0} & 191.1 $\pm$ 34.4  & 50.0 & >1142.8 $\pm$ 69.1  & 0.0 & 1200.0 $\pm$ 0.0  \\
 & 100 & \textbf{100.0} & \textbf{192.8 $\pm$ 27.2} & \textbf{100.0} & 412.3 $\pm$ 63.5 & 0.0 & 1200.0 $\pm$ 0.0 & 0.0 & 1200.0 $\pm$ 0.0\\
 & 200 & \textbf{100.0} & \textbf{401.2 $\pm$ 34.3} & \textbf{100.0} & 839.5 $\pm$ 85.9 & 0.0 & 1200.0 $\pm$ 0.0 & 0.0 & 1200.0 $\pm$ 0.0\\
 & 500 & \textbf{100.0} & \textbf{980.9 $\pm$ 67.7} & 0.0   & 1200.0 $\pm$ 0.0 & 0.0 & 1200.0 $\pm$ 0.0 & 0.0 & 1200.0 $\pm$ 0.0 \\
\midrule
\multirow{5}{*}{2}
 & 10  & \textbf{100.0} & \textbf{19.0 $\pm$ 7.9}   & \textbf{100.0} & 42.4 $\pm$ 17.8   & 95.0 & 54.2 $\pm$ 15.4   & 95.0 & 95.9 $\pm$ 31.7   \\
 & 50  & \textbf{100.0} & \textbf{117.6 $\pm$ 22.4} & \textbf{100.0} & 258.4 $\pm$ 57.9  & 15.0 & 1180.5 $\pm$ 40.4 & 0.0  & 1200.0 $\pm$ 0.0 \\
 & 100 & \textbf{100.0} & \textbf{253.8 $\pm$ 33.2} & \textbf{100.0} & 557.9 $\pm$ 79.9  & 0.0  & 1200.0 $\pm$ 0.0 & 0.0  & 1200.0 $\pm$ 0.0 \\
 & 200 & \textbf{100.0} & \textbf{509.7 $\pm$ 43.1} & 95.0  & >1065.6 $\pm$ 84.3 & 0.0  & 1200.0 $\pm$ 0.0 & 0.0  & 1200.0 $\pm$ 0.0 \\
 & 500 & \textbf{89.0}  & \textbf{1076.4 $\pm$ 51.7} & 0.0   & 1200.0 $\pm$ 0.0 & 0.0  & 1200.0 $\pm$ 0.0 & 0.0  & 1200.0 $\pm$ 0.0 \\
\midrule
\multirow{5}{*}{4}
 & 10  & \textbf{100.0} & \textbf{19.7 $\pm$ 10.1}  & \textbf{100.0} & 53.9 $\pm$ 24.8  & 95.0 & 54.4 $\pm$ 18.6  & 95.0 & 99.0 $\pm$ 38.1  \\
 & 50  & \textbf{100.0} & \textbf{116.9 $\pm$ 22.2} & 100.0 & 313.2 $\pm$ 55.6 & 5.0  & 1182.6 $\pm$ 52.9 & 0.0 & 1200.0 $\pm$ 0.0 \\
 & 100 & \textbf{100.0} & \textbf{235.8 $\pm$ 37.1} & 100.0 & 620.2 $\pm$ 89.7 & 0.0 & 1200.0 $\pm$ 0.0 & 0.0 & 1200.0 $\pm$ 0.0 \\
 & 200 & \textbf{100.0} & \textbf{490.4 $\pm$ 35.9} & 40.0  & >1163.1 $\pm$ 54.1 & 0.0 & 1200.0 $\pm$ 0.0 & 0.0 & 1200.0 $\pm$ 0.0 \\
 & 500 & \textbf{30.0}  & \textbf{1195.4 $\pm$ 13.7} & 0.0   & 1200.0 $\pm$ 0.0 & 0.0 & 1200.0 $\pm$ 0.0 & 0.0 & 1200.0 $\pm$ 0.0 \\
\bottomrule
\end{tabular}
\label{tab:algo_agents_nfz}
\end{table*}

\paragraph{\textbf{Experiment 3: Dynamic Constraint Handling}}
We evaluate robustness under dynamic airspace conditions by introducing temporal NFZs that are active only for specific time intervals. For each environment (5\% obstacles), we fix the obstacle layout and agent scenarios while varying the number of NFZs in {1, 2, 4}. Each NFZ is activated over a random interval within the mission horizon. All results are averaged over 20 independent instances for every configuration. W/O denotes planner without pruning. The results show that \text{DTAPP-IICR} performs the best when the number of NFZs increases compared to other planners (Fig.~\ref{fig:nfz} a). The runtime grows with more activated NFZs, and pruning provides substantial speedups for \text{DTAPP-IICR} up to about 50\% faster than \text{DTAPP-IICR (W/O)} (Fig.~\ref{fig:nfz} b).

\begin{figure}[t]
            \centering
            \includegraphics[width=0.8\linewidth]{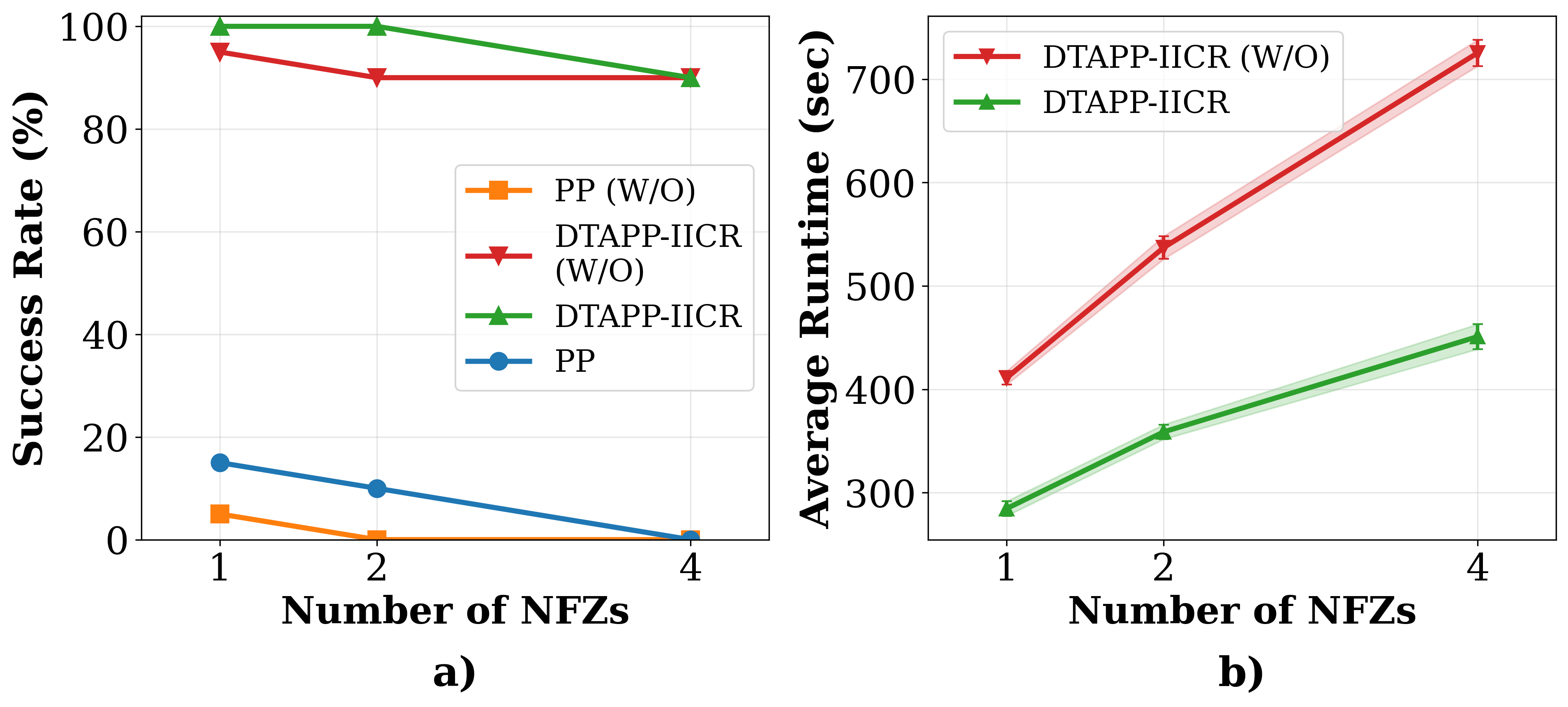}
            \caption{Success Rate and Average Runtime for PP, PP (W/O), DTAPP-IICR, and DTAPP-IICR (W/O). (a) 100 agents with 5\% obstacles, with a time limit of 500 seconds. (b) 1000 agents with 5\% obstacles, with a time limit of 900 seconds.}
            \label{fig:nfz}
        \end{figure}

\paragraph{\textbf{Experiment 4: Realistic Urban Use Case Simulation}}
To demonstrate applicability to real-world conditions, we simulate city-scale UAV delivery scenarios using an urban environment in UNetyEmu simulator~\cite{rodriguez_sbrc_25,rodriguez_3748436}. The environment dimensions are $400 \times 400 \times 25$ with three distributed delivery hubs positioned to create realistic traffic patterns with both intra-hub and inter-hub missions. We evaluate three configurations with \{0, 2, 4\} temporal NFZs, where each NFZ consists of 500-2000 voxels positioned in the upper airspace ($z \geq 12$) to avoid ground obstacles. NFZ activation intervals are randomly distributed between 0-1800 seconds (30 min) with durations of 200-600 seconds (3-10 min), simulating dynamic airspace restrictions. Agent operations span 0-900 seconds (15 min) with heterogeneous properties including speeds of 1.5-3.0 m/s and radius of 1.0-2.0 m. For each configuration, we generate 20 scenarios and test four algorithm variants: PP, PP (W/O), \text{DTAPP-IICR}, and \text{DTAPP-IICR} (W/O) across varying agent counts (10, 50, 100, 200, 500). The time limit is fixed at 1200s (20min).
The results are presented in Table~\ref{tab:algo_agents_nfz}. Without NFZ, \text{DTAPP-IICR} maintains 100\% success rates across all configurations while \text{DTAPP-IICR} (W/O) fails to complete 500-agent operations. In contrast, PP algorithms show significant scalability limitations, achieving only 50\% success at 50 agents with no NFZs and failing completely at large scale, while PP (W/O) fails at 50+ agents across all NFZ configurations. 
As the number of NFZs increases, the runtime of \text{DTAPP-IICR} also increases, reflecting the additional computational effort required for safe navigation around restricted zones. For instance, with 4 NFZs, \text{DTAPP-IICR} achieves a 30\% success rate at 500 agents, indicating that the algorithm remains functional but requires more time. Increasing the allowed runtime, which is feasible in UTM scenarios where operation requests can be submitted to a central controller in advance, would likely improve success rates further.
Overall, the pruning optimization provides substantial improvements, with \text{DTAPP-IICR} consistently outperforming \text{DTAPP-IICR} (W/O) by 2-3 times in runtime. These results indicate that \text{DTAPP-IICR}'s soft constraint handling and iterative conflict resolution approach provides better scalability for large-scale and dynamic urban environments compared to traditional priority-based methods with hard obstacle treatment, which become impractical beyond small-scale deployments.

\section{Conclusion}
This work addressed the critical challenge of scalable preflight planning for heterogeneous UAV fleets operating in shared, dynamic airspace with temporal NFZs. We introduced \text{DTAPP-IICR}, a novel framework that integrates urgency-aware prioritization, iterative conflict resolution, and efficient 4D path planning. Our key contributions include SFIPP-ST, a single-agent planner that handles temporal NFZs and hard obstacles while supporting heterogeneous UAV profiles (varying speeds, sizes, and start times) and modeling inter-agent conflicts through continuous geometric cost functions. Furthermore, our iterative conflict resolution mechanism leverages LNS guided by geometric conflict graphs to systematically resolve residual conflicts, while directional pruning accelerates search in 3D grids by up to 50\% without sacrificing completeness. 

Experimental validation across Monte Carlo simulations and realistic urban scenarios demonstrates that \text{DTAPP-IICR}'s superior scalability and robustness compared to the previous batch CBS/ECBS method in the UTM context. These results establish \text{DTAPP-IICR} as a foundational solution for safe, efficient UAV integration in dense urban airspace under regulatory constraints.

While our current method prioritizes finding conflict-free solutions, future works will explore advanced neighborhood selection strategies inspired by LNS~\cite{LNS, LNS-ORI, LNS-comb}. This includes designing a set of heuristics tailored for 4D environments and implementing an adaptive learning mechanism, such as a multi-armed bandit~(MAB), to dynamically choose the most effective heuristic during the search~\cite{Balance24, ADDRESS}. This will allow the planner to not only ensure safety but also optimize for operational efficiency. Another promising direction would be to integrate SDP/HMAPP~\cite{hmapp,sdp} techniques that decompose the environment into smaller sub-regions, enabling scalable planning across very large fleets.

\section*{Acknowledgments}
This work is supported in part by Sweden’s Innovation Agency Vinnova (Projects 2022-02671 DyMuDRoP and  2024-01322 AHA-IMPUT), the Excellence Center at Linköping-Lund in Information Technology (ELLIIT), Ericsson Telecomunicações Ltda., and the São Paulo Research Foundation (FAPESP), grant 2021/00199-8, CPE SMARTNESS.

\bibliographystyle{unsrt}  
\bibliography{references}  

@book{EASA,
    author = {European Union Aviation Safety Agency},
    title = {Commission Implementing Regulation (EU) 2019/947 of 24 May 2019 on the rules and procedures for the operation of unmanned aircraft},
    publisher = {Official Journal of the European Union},
    year = {2019},
    url={http://data.europa.eu/eli/reg_impl/2019/947/oj},
    address = {European Union}
}

@book{EASA1,
    author = {European Union Aviation Safety Agency},
    title = {Commission Implementing Regulation (EU) 2021/664 of 22 April 2021 on a regulatory framework for the U-space},
    publisher = {Official Journal of the European Union},
    year = {2021},
    url = {http://data.europa.eu/eli/reg_impl/2021/664/oj},
    address = {European Union}
}

@book{SESAR,
    author = {Single European Sky ATM Research 3 Joint Undertaking},
    title = {European drones outlook study – Unlocking the value for Europe},
    publisher = {Publications Office of the European Union},
    year = {2017},
    doi = {10.2829/085259},
    address = {European Union}
}

@article{CBS,
    title = {Conflict-based search for optimal multi-agent pathfinding},
    journal = {Artificial Intelligence},
    volume = {219},
    pages = {40-66},
    year = {2015},
    issn = {0004-3702},
    doi = {https://doi.org/10.1016/j.artint.2014.11.006},
    url = {https://www.sciencedirect.com/science/article/pii/S0004370214001386},
    author = {Guni Sharon and Roni Stern and Ariel Felner and Nathan R. Sturtevant}
}

@inproceedings{Extended-MAPF,
    author = {Ho, Florence and Salta, Ana and Geraldes, Ruben and Goncalves, Artur and Cavazza, Marc and Prendinger, Helmut},
    title = {Multi-Agent Path Finding for UAV Traffic Management},
    year = {2019},
    isbn = {9781450363099},
    publisher = {International Foundation for Autonomous Agents and Multiagent Systems},
    address = {Richland, SC},
    booktitle = {Proceedings of the 18th International Conference on Autonomous Agents and MultiAgent Systems},
    pages = {131–139},
    numpages = {9},
    doi={10.5555/3306127.3331684}
}

@inproceedings{LNS,
  title     = {Anytime Multi-Agent Path Finding via Large Neighborhood Search},
  author    = {Li, Jiaoyang and Chen, Zhe and Harabor, Daniel and Stuckey, Peter J. and Koenig, Sven},
  booktitle = {Proceedings of the Thirtieth International Joint Conference on Artificial Intelligence (IJCAI)},
  pages     = {4127--4135},
  publisher = {Association for the Advancement of Artificial Intelligence (AAAI)},
  address = {Montreal, Canada},
  year      = {2021},
  month     = {8},
  doi       = {10.24963/ijcai.2021/568},
  url       = {https://doi.org/10.24963/ijcai.2021/568},
}

@article{LNS-comb,
    title = {A survey of adaptive large neighborhood search algorithms and applications},
    journal = {Computers \& Operations Research},
    volume = {146},
    pages = {105903},
    year = {2022},
    issn = {0305-0548},
    doi = {https://doi.org/10.1016/j.cor.2022.105903},
    url = {https://www.sciencedirect.com/science/article/pii/S0305054822001654},
    author = {Setyo Tri {Windras Mara} and Rachmadi Norcahyo and Panca Jodiawan and Luluk Lusiantoro and Achmad Pratama Rifai}
}

@article{MLLNS, 
    title={Anytime Multi-Agent Path Finding via Machine Learning-Guided Large Neighborhood Search}, 
    volume={36}, 
    url={https://ojs.aaai.org/index.php/AAAI/article/view/21168}, 
    DOI={10.1609/aaai.v36i9.21168}, 
    number={9}, 
    journal={Proceedings of the AAAI Conference on Artificial Intelligence}, 
    author={Huang, Taoan and Li, Jiaoyang and Koenig, Sven and Dilkina, Bistra}, 
    year={2022}, 
    month={Jun.}, 
    pages={9368-9376} 
}

@Inbook{MAPF,
    author="Stern, Roni",
    title="Multi-Agent Path Finding -- An Overview",
    bookTitle="Artificial Intelligence: 5th RAAI Summer School, Dolgoprudny, Russia, July 4--7, 2019, Tutorial Lectures",
    year="2019",
    publisher="Springer International Publishing",
    address="Cham",
    pages="96--115",
    isbn="978-3-030-33274-7",
    doi="10.1007/978-3-030-33274-7_6",
    url="https://doi.org/10.1007/978-3-030-33274-7_6"
}

@InProceedings{LNS-ORI,
    author="Shaw, Paul",
    editor="Maher, Michael
    and Puget, Jean-Francois",
    title="Using Constraint Programming and Local Search Methods to Solve Vehicle Routing Problems",
    booktitle="Principles and Practice of Constraint Programming --- CP98",
    year="1998",
    publisher="Springer Berlin Heidelberg",
    address="Berlin, Heidelberg",
    pages="417--431",
    isbn="978-3-540-49481-2",
    doi="https://doi.org/10.1007/3-540-49481-2_30"
}

@INPROCEEDINGS{SIPP,
  author={Phillips, Mike and Likhachev, Maxim},
  booktitle={2011 IEEE International Conference on Robotics and Automation}, 
  title={SIPP: Safe interval path planning for dynamic environments}, 
  publisher = {IEEE},
  address = {Shanghai, China},
  year={2011},
  volume={},
  number={},
  pages={5628-5635},
  doi={10.1109/ICRA.2011.5980306}
}

@inproceedings{ICBS,
    author = {Eli Boyarski and Ariel Felner and Roni Stern and Guni Sharon and David Tolpin and Oded Betzalel and Eyal Shimony},
    title = { ICBS: Improved Conflict-Based Search Algorithm for Multi-Agent Pathfinding},
    booktitle = {Proceedings of the Twenty-Fourth International Joint Conference on Artificial Intelligence (IJCAI)},
    publisher = {IJCAI},
    address = {South America},
    pages={740-746},
    year = {2015},
    url = {https://www.ijcai.org/Proceedings/15/Papers/110.pdf}
}

@article{EECBS,
    title={EECBS: A Bounded-Suboptimal Search for Multi-Agent Path Finding}, 
    volume={35}, 
    url={https://ojs.aaai.org/index.php/AAAI/article/view/17466}, 
    DOI={10.1609/aaai.v35i14.17466},
    number={14}, 
    journal={Proceedings of the AAAI Conference on Artificial Intelligence}, 
    author={Li, Jiaoyang and Ruml, Wheeler and Koenig, Sven}, 
    year={2021}, 
    month={May}, 
    pages={12353-12362} 
}

@inproceedings{ECBS,
    author = {Max Barer and Guni Sharon and Roni Stern and Ariel Felner},
    title = {Suboptimal Variants of the Conflict-Based Search Algorithm for the Multi-Agent Pathfinding Problem},
    booktitle = {Proceedings of the Seventh Annual Symposium on Combinatorial Search (SoCS)},
    publisher = {SoCS},
    address = {NLD},
    pages = {961–962},
    year = {2014},
    url = {https://doi.org/10.1609/socs.v5i1.18315}
}

@article{LNS2, 
    title={MAPF-LNS2: Fast Repairing for Multi-Agent Path Finding via Large Neighborhood Search}, 
    volume={36}, 
    url={https://ojs.aaai.org/index.php/AAAI/article/view/21266}, 
    DOI={10.1609/aaai.v36i9.21266}, 
    number={9}, 
    journal={Proceedings of the AAAI Conference on Artificial Intelligence}, 
    author={Li, Jiaoyang and Chen, Zhe and Harabor, Daniel and Stuckey, Peter J. and Koenig, Sven}, 
    year={2022}, 
    month={Jun.}, 
    pages={10256-10265} 
}

@article{ADDRESS, 
    title={Anytime Multi-Agent Path Finding with an Adaptive Delay-Based Heuristic}, 
    volume={39}, 
    url={https://ojs.aaai.org/index.php/AAAI/article/view/34495}, 
    DOI={10.1609/aaai.v39i22.34495}, 
    number={22}, 
    journal={Proceedings of the AAAI Conference on Artificial Intelligence}, 
    author={Phan, Thomy and Zhang, Benran and Chan, Shao-Hung and Koenig, Sven}, 
    year={2025}, 
    month={Apr.}, 
    pages={23286-23294}
}

@book{UTM-USA,
    author = {Federal Administration Aviation},
    title = {FAA Aerospace Forecast Fiscal Years 2023–2043},
    publisher = {Federal Administration Aviation},
    url = {https://www.faa.gov/dataresearch/aviation/aerospaceforecasts/faa-aerospace-forecast-fy-2023-2043},
    year = {2023},
    address = {Washington, DC}
}

@book{UTM-USA1,
    author = {FAA Aerospace Forecast Fiscal Years 2024–2044},
    title = {Emerging Aviation Entrants: Unmanned Aircraft System and Advanced Air Mobility},
    publisher = {FAA Aerospace Forecast},
    url = {https://www.faa.gov/dataresearch/aviation/aerospaceforecasts/uas-and-aam.pdf},
    year = {2024},
    address = {Washington, DC}
}

@article{UTM-CHINA,
    author = {Civil Aviation Administration of China},
    title = {Statistical Bulletin of Civil Aviation Industry Development in 2023},
    journal = {China Civil Aviation Annual Report},
    pages = {12},
    number = {1},
    volume = {1},
    year = {2023},
    url = {https://www.caac.gov.cn/English/Research/Reports/Statistical/202412/t20241211_226085.html}
}

@article{Balance24, 
    title={Adaptive Anytime Multi-Agent Path Finding Using Bandit-Based Large Neighborhood Search}, 
    volume={38}, 
    url={https://ojs.aaai.org/index.php/AAAI/article/view/29701}, 
    DOI={10.1609/aaai.v38i16.29701}, 
    number={16}, 
    journal={Proceedings of the AAAI Conference on Artificial Intelligence}, 
    author={Phan, Thomy and Huang, Taoan and Dilkina, Bistra and Koenig, Sven}, 
    year={2024}, 
    month={Mar}, 
    pages={17514-17522} 
}

@INPROCEEDINGS{PP,
    author={Erdmann, M. and Lozano-Perez, T.},
    booktitle={Proceedings. 1986 IEEE International Conference on Robotics and Automation}, 
    title={On multiple moving objects}, 
    year={1986},
    volume={3},
    number={},
    publisher = {IEEE},
    address = {San Francisco, CA},
    pages={1419-1424},
    doi={10.1109/ROBOT.1986.1087401}
}

@article{ma2019searching, 
    title={Searching with Consistent Prioritization for Multi-Agent Path Finding}, 
    volume={33}, 
    url={https://ojs.aaai.org/index.php/AAAI/article/view/4758}, 
    DOI={10.1609/aaai.v33i01.33017643}, 
    number={01}, 
    journal={Proceedings of the AAAI Conference on Artificial Intelligence}, 
    author={Ma, Hang and Harabor, Daniel and Stuckey, Peter J. and Li, Jiaoyang and Koenig, Sven}, 
    year={2019}, 
    month={Jul.}, 
    pages={7643-7650} 
}

@inproceedings{rodriguez_sbrc_25,
    author = {Mauricio Rodriguez Cesen and Ariel Góes de Castro and Ibini Santana and Ramon Fontes and Fabricio R. Cesen and Christian Esteve Rothenberg},
    title = { UNetyEmu: Unity-based simulator for aerial and non-aerial vehicles with integrated network emulation},
    booktitle = {Anais Estendidos do XLIII Simpósio Brasileiro de Redes de Computadores e Sistemas Distribuídos},
    location = {Natal/RN},
    year = {2025},
    issn = {2177-9384},
    pages = {100--111},
    publisher = {SBC},
    address = {Porto Alegre, RS, Brasil},
    doi = {10.5753/sbrc_estendido.2025.7122},
    url = {https://sol.sbc.org.br/index.php/sbrc_estendido/article/view/35864}
}

@inproceedings{rodriguez_3748436,
    author = {Rodriguez, Mauricio and de Castro, Ariel Goes and Fontes, Ramon and Rodriguez, Fabricio and Rothenberg, Christian},
    title = {An Integrated Framework for Network Emulation and Multi-vehicle Algorithm Testing},
    year = {2025},
    isbn = {9798400720260},
    publisher = {Association for Computing Machinery},
    address = {New York, NY, USA},
    url = {https://doi.org/10.1145/3744969.3748436},
    doi = {10.1145/3744969.3748436},
    booktitle = {Proceedings of the ACM SIGCOMM 2025 Posters and Demos},
    pages = {167–169},
    numpages = {3},
    location = {Coimbra, Portugal},
    series = {ACM SIGCOMM Posters and Demos '25}
}

@article{hmapp, 
    title={A Hierarchical Approach to Multi-Agent Path Finding}, 
    volume={12}, 
    url={https://ojs.aaai.org/index.php/SOCS/article/view/18586}, 
    DOI={10.1609/socs.v12i1.18586}, 
    number={1}, 
    journal={Proceedings of the International Symposium on Combinatorial Search}, 
    author={Zhang, Han and Yao, Mingze and Liu, Ziang and Li, Jiaoyang and Terr, Lucas and Chan, Shao-Hung and Kumar, T. K. Satish and Koenig, Sven}, 
    year={2021}, 
    month={Jul}, 
    pages={209-211} 
}

@article{sdp,
    author = {Wilt, Christopher and Botea, Adi},
    year = {2014},
    month = {05},
    pages = {332-340},
    title = {Spatially Distributed Multiagent Path Planning},
    volume = {24},
    journal = {Proceedings of the International Conference on Automated Planning and Scheduling},
    doi = {10.1609/icaps.v24i1.13618}
}

@inproceedings{RCCBS,
    author="Chen, Ming
    and He, Ning
    and Hong, Chen",
    title="Continuous Multi-Agent Path Finding for Drone Delivery",
    booktitle="Proceedings of the 2024 IEEE/CVF Conference on Computer Vision and Pattern Recognition (PRCV)",
    publisher = {Springer},
    address = {Singapore},
    year="2024",
    pages="129--141",
    url={https://doi.org/10.1007/978-981-97-8795-1_9}
}

@article{MAPF-HR,
    title = {Multi-Agent Path Finding with heterogeneous edges and roundtrips},
    journal = {Knowledge-Based Systems},
    volume = {234},
    pages = {107554},
    year = {2021},
    issn = {0950-7051},
    doi = {https://doi.org/10.1016/j.knosys.2021.107554},
    url = {https://www.sciencedirect.com/science/article/pii/S0950705121008169},
    author = {Bing Ai and Jiuchuan Jiang and Shoushui Yu and Yichuan Jiang}
}

@article{TakeoffSpeedAdjutement,
  title={Multi-Agent Path Finding in Unmanned Aircraft System Traffic Management With Scheduling and Speed Variation
},
  author={Florence Ho and Artur Alves Gonçalves and Bastien Rigault and Ruben Geraldes and Alexandre Chicharo and Marc Cavazza and Helmut Prendinger},
  journal={IEEE Intelligent Transportation Systems Magazine},
  year={2021},
  volume={14},
  pages={2--15},
}

@article{lastMilesDrone,
  title={Last-Mile Drone Delivery: Past, Present, and Future},
  author={Hossein Eskandaripour and  Enkhsaikhan Boldsaikhan},
  journal={Drones},
  volume={2},
  pages = {77},
  number={7},
  year={2023},
  url={https://doi.org/10.3390/drones7020077}
}

@article{DroneAgri,
  title={Drones in Precision Agriculture: A Comprehensive Review of Applications, Technologies, and Challenges},
  author={Ridha Guebsi and Sonia Mami and Karem Chokmani},
  journal={Drones},
  pages = {686},
  volume={8},
  number={11},
  year={2024},
  url={https://doi.org/10.3390/drones8110686}
}

@Article{SearchAndRescue,
    AUTHOR = {Lyu, Mingyang and Zhao, Yibo and Huang, Chao and Huang, Hailong},
    TITLE = {Unmanned Aerial Vehicles for Search and Rescue: A Survey},
    JOURNAL = {Remote Sensing},
    PAGES = {3266},
    VOLUME = {15},
    YEAR = {2023},
    NUMBER = {13},
    ARTICLE-NUMBER = {3266},
    URL = {https://www.mdpi.com/2072-4292/15/13/3266},
    ISSN = {2072-4292},
    DOI = {10.3390/rs15133266}
}

@article{Surveillance,
  title={RF/WiFi-based UAV surveillance systems: A systematic literature review},
  author={Igor Bisio and Chiara Garibotto and Halar Haleem and Fabio Lavagetto and Andrea Sciarrone},
  journal = {Internet of Things},
  pages = {101201},
  volume = {26},
  year={2024},
  url={https://doi.org/10.1016/j.iot.2024.101201}
}

@inproceedings{MAPD-Ma,
    author = {Ma, Hang and Li, Jiaoyang and Kumar, T.K. Satish and Koenig, Sven},
    title = {Lifelong Multi-Agent Path Finding for Online Pickup and Delivery Tasks},
    booktitle = {Proceedings of the 16th Conference on Autonomous Agents and MultiAgent Systems},
    year = {2017},
    publisher = {International Foundation for Autonomous Agents and Multiagent Systems},
    address = {Richland, SC},
    pages = {837–845},
    numpages = {9},
    location = {S\~{a}o Paulo, Brazil},
    series = {AAMAS '17}
}

\appendix

\section{Simulation}

\subsection{Simulator}

UNetyEmu is a framework that allows realistic and scalable experiments to be carried out with multiple vehicles (e.g., drones and cars) at the same time, integrating Unity's high-fidelity 3D simulation with Mininet-WiFi's real-time network emulation. The system enables the simultaneous simulation of multiple autonomous vehicles, both aerial and terrestrial, each with different characteristics, operating under dynamic communication conditions typical of smart city environments. This combination offers researchers a powerful and customizable tool for evaluating coordination strategies, network performance, and algorithms related to drone behavior, such as obstacle avoidance, route planning, and coverage, facilitating advanced studies in areas such as 5G vehicular communication, edge computing, and intelligent transportation systems \cite{rodriguez_sbrc_25,rodriguez_3748436}.

\subsection{City scale scenario}

In this work, we use the UNetyEmu simulator for external validation of our DTAPP-IICR algorithm. To do this, we consider that the drones have no communication and are fully autonomous, following a specific planned route. We used a 3D map at urban scale with three designated UAV centers, random delivery destinations on a $400 \times 400 \times 25$ map, and mobility restrictions such as No Fly Zones (NFZ). Figure~\ref{fig:UNetyEmu2} 
shows a screenshot of the simulator during the evaluation of the proposed algorithm using 100 drones in the scene, in which you can see all the routes taken by the drones during the simulation. In addition, Figure~\ref{fig:UNetyEmu1} shows another screenshot of the simulator in which a blue drone delivering a package can be seen at the bottom left. In front of it, on the ground, is a red DronePad that is assigned to deliver the package. Meanwhile, to its right, you can see a green drone with different physical characteristics that is making another delivery. Also, at the top of the image, you can see a red rectangle representing the NFZ during the simulation. You can find this demo at \url{https://github.com/amathsow/4DPlanning}.

\begin{figure}[t]
    \centering
    \includegraphics[width=0.7\linewidth]{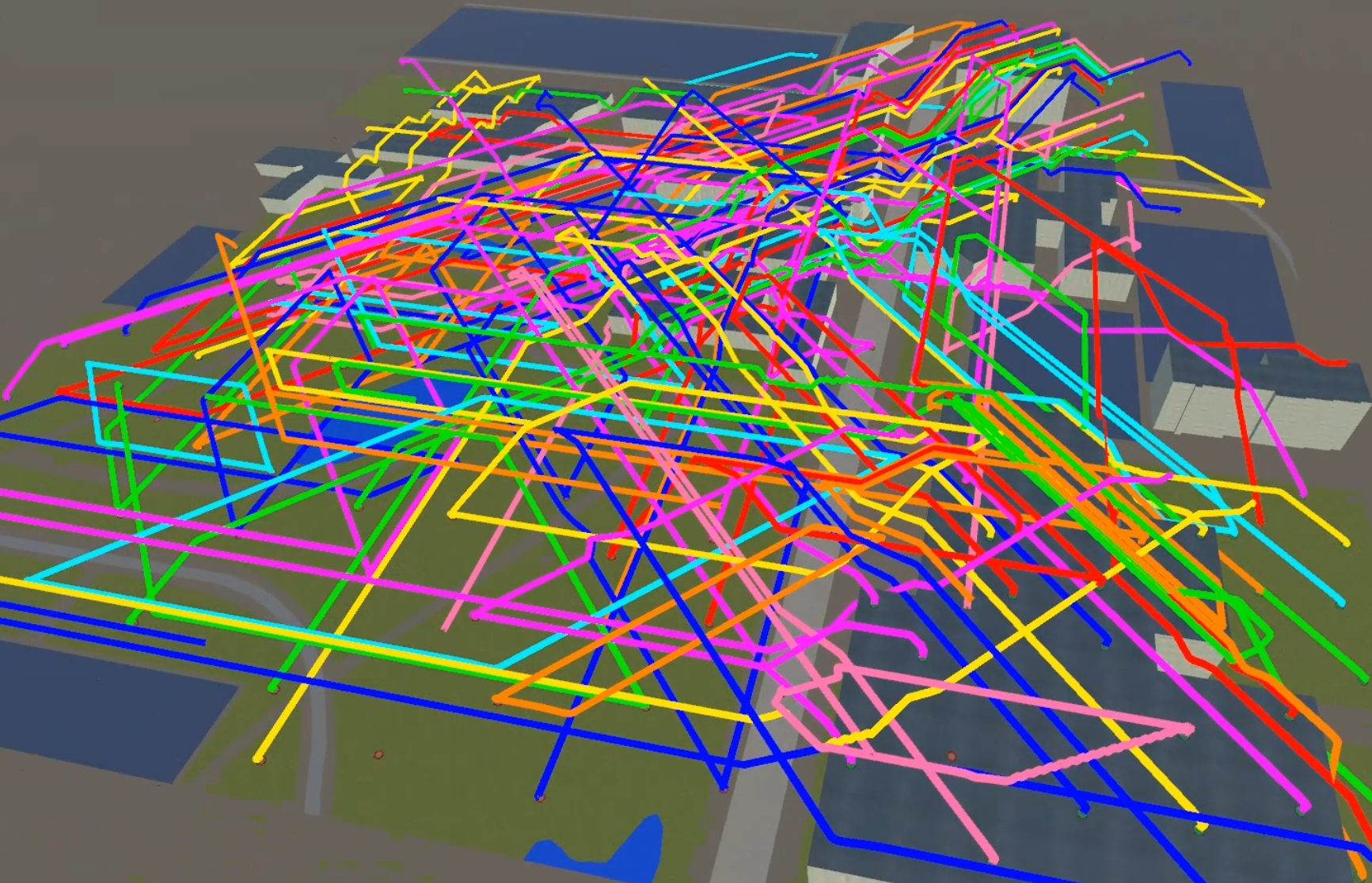}
    \caption{100 drone paths in an urban scene are shown in different colors when evaluating our DTAPP-IICR algorithm.}
    \label{fig:UNetyEmu2}
\end{figure}

\begin{figure}[t]
    \centering
    \includegraphics[width=0.7\linewidth]{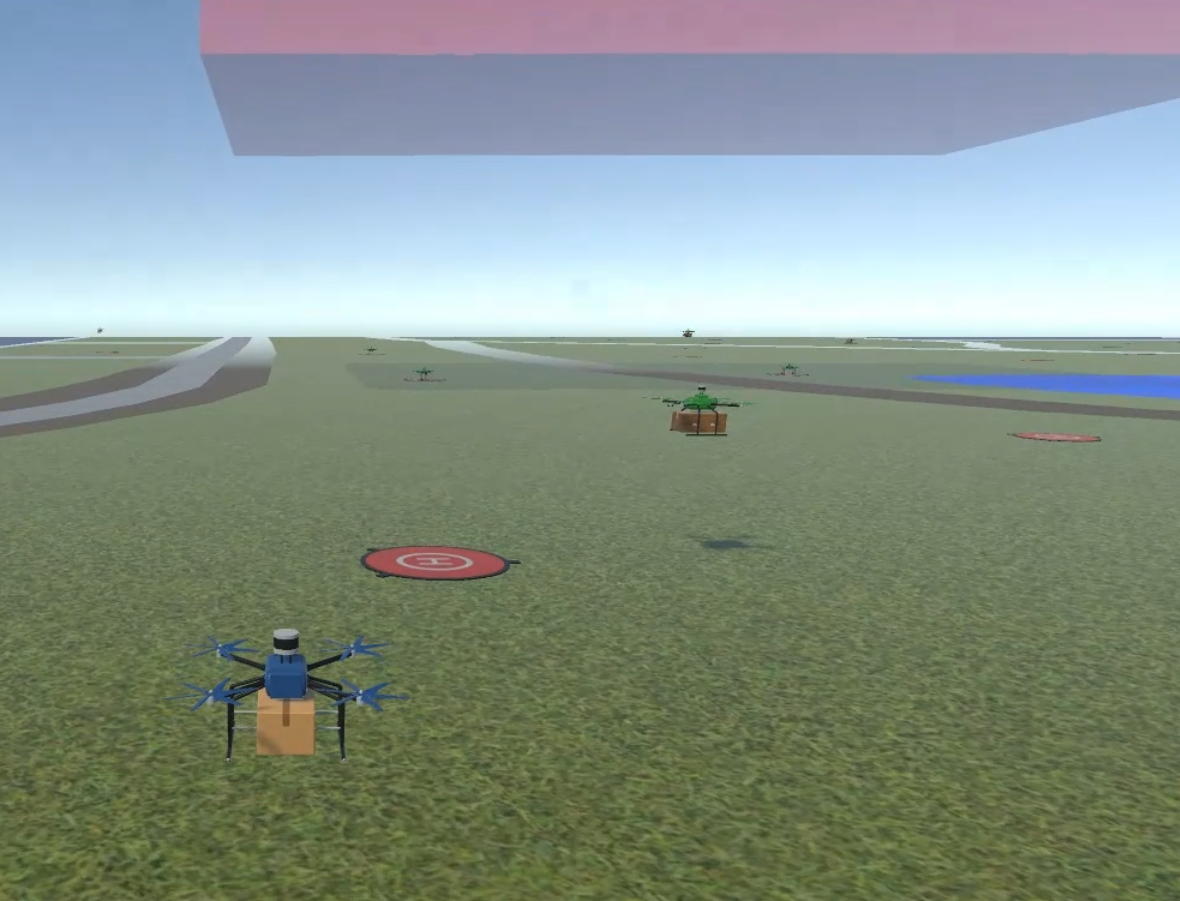}
    \caption{Different drones on the scene during the delivery of a package with mobility restrictions such as NFZs.}
    \label{fig:UNetyEmu1}
\end{figure}

\end{document}